\documentclass[11pt]{amsart}

\usepackage{verbatim, amssymb,hyperref}

\usepackage{accents}
\usepackage{color}
\usepackage{soul}
\usepackage{bbm}
\usepackage{graphicx}
\usepackage{amsmath}

\newcommand{\red}{\color[rgb]{1,0,0}}

\newcommand{\eps}{\varepsilon}

\newcommand{\orient}{\mathrm{orient}}

\newcommand{\cF}{\mathcal F}
\newcommand{\cC}{\mathcal C}

\newcommand{\cE}{\mathcal E}

\newcommand{\cO}{\mathcal O}

\newcommand{\grad}{\nabla}

\newcommand{\spann}{\operatorname{span}}

\newcommand{\IW}{\mathbb W}

\newcommand{\cW}{\mathcal W}
\newcommand{\sign}{\mathrm{sign}}
\newcommand{\vecn}{\mathfrak n}
\newcommand{\vecm}{\mathfrak m}

\newcommand{\II}{\mathbb I}

\newcommand{\1}{1\hspace{-0.098cm}\mathrm{l}}

\newcommand{\N}{{\mathbb N}}

\newcommand{\R}{{\mathbb R}}

\newcommand{\cN}{{\mathcal N}}

\newcommand{\Hess}{\text{Hess}\,}

\newcommand{\cP}{{\mathcal P}}

\theoremstyle{plain}
\newtheorem{theorem}{Theorem}[section]

\newtheorem{lemma}[theorem]{Lemma}

\newtheorem{defi}[theorem]{Definition}

\theoremstyle{definition}
\newtheorem{rem}[theorem]{Remark}
\newtheorem{exa}[theorem]{Example}

\makeatletter
\@namedef{subjclassname@2020}{%
	\textup{2020} Mathematics Subject Classification}
\makeatother

\begin{document}

\title[On minimal representations of shallow ReLU networks]%
{On minimal representations of shallow ReLU networks}

\author[]
{Steffen Dereich}
\address{Steffen Dereich\\
	Institut f\"ur Mathematische Stochastik\\
	Fachbereich 10: Mathematik und Informatik\\
	Westf\"alische Wilhelms-Universit\"at M\"unster\\
	Orl\'eans-Ring 10\\
	48149 M\"unster\\
	Germany}
\email{steffen.dereich@wwu.de}

\author[]
{Sebastian Kassing}
\address{Sebastian Kassing\\
	Institut f\"ur Mathematische Stochastik\\
	Fachbereich 10: Mathematik und Informatik\\
	Westf\"alische Wilhelms-Universit\"at M\"unster\\
	Orl\'eans-Ring 10\\
	48149 M\"unster\\
	Germany}
\email{sebastian.kassing@wwu.de}

\keywords{Neural networks, shallow networks, minimal representations, ReLU activation}
\subjclass[2020]{Primary 68T05; Secondary 68T07, 26B40}

\begin{abstract}
	The realization function of a shallow ReLU network is a continuous and piecewise affine function $f:\R^d\to \R$, where the domain $\R^{d}$ is partitioned by a set of $n$ hyperplanes into cells on which $f$ is affine. We show that the minimal representation for $f$ uses either $n$, $n+1$ or $n+2$ neurons and we characterize each of the three cases. In the particular case, where the input layer is one-dimensional, minimal representations always use at most $n+1$ neurons but in all higher dimensional settings there are functions for which  $n+2$ neurons are needed. Then we show that the set of minimal networks representing $f$ forms a $C^\infty$-submanifold $M$ and we derive the dimension and the number of connected components of $M$. Additionally, we give a criterion for the hyperplanes that guarantees that all continuous, piecewise affine functions are  realization functions of appropriate ReLU networks.
\end{abstract}

\maketitle
\section{Introduction}
Standard convergence results for Stochastic Gradient Descent schemes often use strong convexity of the loss function around an isolated minimum. However, a common loss function in neural network applications depends exclusively on the realization function generated by the weights and biases and there exists a non-trivial set of parametrizations generating the same realization function. In practice, the most frequently used activation function is the ReLU function $\sigma(x)=\max(x,0)$. The positive homogeneity of ReLU leads to many redundancies in the loss landscape (\cite{petzka2020notes}). As is well known (see for instance \cite{qin2020reducing}) one can restrict attention to networks where the vector of  weights associated to each hidden neuron have uniform norm without reducing the set of representable functions.  
Additionally, \cite{cooper2018loss} showed that in the overparameterized regime (more parameters in the network architecture than data points to fit), under the assumption that the network is able to fit the data points and the activation function~$\sigma$ is smooth, the set of parameters $P \subset \R^n$ achieving zero training loss forms a $(n-p)$-dimensional smooth submanifold for almost all input data $S=((x_i,y_i))_{1=1, \dots, p}\subset \R^{(d+1) \times p}$ with fixed $p \in \N$. Moreover, for every point $x$ in the manifold the Hessian of the loss function $\Hess F(x)$ has $p$ positive eigenvalues and $n-p$ eigenvalues equal to $0$.
\cite{FGJ2019} and \cite{Der19Central} established convergence rates for SGD and \cite{moitra2020fast} for the Langevin diffusion under these assumptions for the critical locus. See also \cite{wojtowytsch2021stochastic} for results on SGD in continuous time in the case where the perturbation scales with the value of the loss function.

Nonetheless, general networks do not meet this assumption. For an easy counterexample, we try to approximate the zero function $f: \R \to \R; x \mapsto 0$ with a neural network using one hidden-layer with one neuron.
Such a network $\IW$ with weights $w_{1,1}^1,w_{1,1}^2$ and biases $b_1^1,b_1^2$ responds to an input $x$ via 
$$
		\mathfrak N^\IW(x)=w_{1,1}^2 \max(w_{1,1}^1x+b_1^1,0)+b_1^2.
$$
Now, the equation $\mathfrak N^\IW(x) \equiv 0$ is solved by the set 
$$
\bigl\{(w_{1,1}^1,w_{1,1}^2,b_1^1,b_1^2) \in \R^4: \bigl( w_{1,1}^1 = 0 \vee w_{1,1}^2=0\bigr) , b_1^2=-w_{1,1}^2 b_1^1\bigr\}
$$
which around $(0,0,0,0)$ forms not even locally a submanifold of $\R^4$. 

One key to solving this problem is minimality. We classify the set of parametrizations leading to the same realization function in shallow ReLU networks under the assumption that the network is minimal, meaning that there exists no neural network with less neurons in the hidden-layer which admits the same realization function, and show that the set of minimal parametrizations forms a smooth submanifold. We distinguish between three different scenarios and derive the dimension and the number of connected components in each of the situations.

The idea of minimal networks has been used by \cite{sussmann1992uniqueness} to describe the set of configurations with the same structure and realization function for real-valued neural networks with one hidden-layer and activation function $\sigma=\tanh$. Therein, the point-symmetry of the activation function is used to state that the set of minimal networks with the same input-output map is the union of discrete points generated by a finite group of transformations. This result has been generalized in \cite{Nitta2003Uniqueness} for complex-valued neural networks without bias and in \cite{Kobayashi2010Excpetional} for general complex-valued networks with hyperbolic tangent activation. In the present work the transformation group is not countable, due to the positive-homogeneity of the ReLU function, leading to a rich set of configurations. 

In a similar but independent approach, \cite{cheridito2021landscape} characterized the critical points of the loss function in the task of approaching affine functions on an interval with respect to the $L^2$-norm. However, we consider all functions that are representable by neural networks, i.e. piecewise affine functions admitting some additional properties, and allow the input dimension to be arbitrarily large.

The second main result of this article concerns the class of affine functions that are representable by shallow ReLU nets. We show that all continuous and piecewise affine functions can be constructed by a neural network if the hyperplanes which partition the space into cells on which the function acts affine (the \emph{breaklines}) do not intersect in an artificial way (see further Theorem~\ref{lem:dim}). Moreover, we show that for a fixed number of hidden neurons, the set of weights and biases that lead to those piecewise affine realizations not covered by Theorem~\ref{lem:dim} are a nullset with respect to the Lebesgue-measure. Therefore, the set of piecewise affine functions satisfying the assumptions of Theorem~\ref{lem:dim} have universal approximation properties (\cite{LESHNO1993861}, \cite{HORNIK19931069}).

There are a lot of results concerning the approximation power of shallow nets (e.g. \cite{cybenko1989approximation}, \cite{hornik1989multilayer}, \cite{LESHNO1993861} and \cite{HORNIK19931069}). \cite{Hanin2019Universal} shows that neural nets with arbitrary depths are able to approximate every continuous function $f:[0,1]^{d} \to \R$, if there are at least $d+2$ neurons on each hidden layer. Further, \cite{arora2018understanding} showed that every piecewise linear function $f:\R^{d} \to \R$ can be represented by a ReLU neural network with at most $\lceil \log_2(n+1) \rceil+1$ depth and for every piecewise linear, univariate function $f:\R \to \R$ with $n$ breakpoints there exists a net having one hidden-layer with $n+1$ neurons that represents $f$. We give sufficient conditions on the piecewise affine function in order to be representable by a shallow neural net and generalize their second result for multivariate functions $f:\R^{d} \to \R$. We show that if $f$ has $n$ breaklines and is representable by a shallow neural net there exists a net with $n+2$ neurons on the hidden layer that represents $f$. Further, we give algebraic conditions on $f$ that are equivalent to the fact that $f$ is representable by a neural net with $n$ hidden neurons (resp. $n+1$ hidden neurons) and characterize the manifold of minimal representations for $f$ depending on the minimal number of neurons.

\section{Shallow ReLU networks}

In this section we give the basic definitions and notions of a ReLU network with one hidden layer. We denote by $d_0 \in \N$ the input dimension, by $d_1\in\N$ the numbers of neurons in the hidden layer and by  $(\cdot)_+:\R\to \R$ the ReLU activation function, i.e. $(x)_+=\max(x,0)$, for all $x \in \R$. For simplicity we let the output dimension $d_2=1$.
The configuration of such a neural network is described by a tuple
$$
\IW=(W^1,b^1, W^2, b^2) \in \R^{d_{1}\times d_{0}}\times \R^{d_1}\times \R^{1\times d_{1}}\times \R  
=:\cW^{d_0,d_1,1},
$$
where, for $i=1,2$, we always write 
$$
W^i=(w_{j,k}^i)_{\substack{j=1, \dots, d_i \\ k=1, \dots, d_{i-1}}} \quad  \text{ and } \quad   b^i=(b_1^i, \dots, b_{d_i}^i)^\dagger,
$$
where $^\dagger$ denotes the transpose of a matrix or a vector.
Moreover, for $j=1, \dots,d_1$, we write $w_j^1=(w_{j,1}^1, \dots, w_{j,d_0}^1)^\dagger$ and $w_j^2=w_{1,j}^2$.

We often refer to a configuration of a neural network as the (\emph{neural}) \emph{network}~$\IW$. A configuration $\IW \in \cW^{d_0, d_1, 1}$ describes a function $\mathfrak N^\IW: \R^{d_0} \to \R$ via
\begin{align} \label{eq:response}
	\mathfrak N^\IW(x)= b^2 + \sum_{j=1}^{d_1} w_{j}^2 \max\Bigl(\sum_{i=1}^{d_0} w_{j,i}^1x_i+b_j^1, 0\Bigr).
\end{align}
We call $\mathfrak N^\IW$ \emph{realization function} or \emph{response}. Further, we call $f \in \mathcal F = \bigcup_{d_1\in\N}\{\mathfrak N^\IW : \IW \in \cW^{d_0,d_1,1}\}$ a \emph{representable function}. 

The above network representation has a lot of redundancies. Moreover, there are more intuitive geometric descriptions of the response function. To understand this we first introduce  an orientation on $\R^{d_0}$. The explicit choice will have no relevance.

We fix a convex cone $\cO_+\subset \R^{d_0}\backslash\{0\}$ such that for $\cO_-=-\cO_+$ one has $\R^{d_0}\backslash \{0\}=\cO_+\dot \cup \,\cO_-$ and we call the vectors in $\cO_+$ positively oriented  and the vectors in $\cO_-$ negatively oriented. $\mathrm{orient}:\R^{d_0}\backslash\{0\}\to\{\pm1\}$ denotes the mapping that assigns $x$ its orientation $\mathrm{orient}(x)$.

We call a network $\IW=\cW^{d_0,d_1,1}$ \emph{non-degenerate} iff for all $j=1,\dots,d_1$, $w_{j}^2w_j^1\not=0$. Now we will give a different description of the response of a non-degenerate network $\IW$:
 we say that the neuron $j\in\{1,\dots, d_1\}$ has
\begin{itemize} 
\item positive normal  ${\displaystyle \vecn_j=\frac {\orient(w_j^1)  }{|w_j^1|} w_j^1\in \cO_+^1:=\{x\in\cO_+: |x|=1\}}$,
\item offset $o_j=-\frac {\orient(w_j^1)}{|w_j^1|} b_j^1\in\R$,
\item kink $\Delta_j=|w_j^1| w_j^2\in\R\backslash\{0\}$ and
\item orientation $\sigma_j=\mathrm{orient}(w_j^1)\in\{\pm1\}$.
\end{itemize}
For $d_1 \in \N$ we call $(\vecn, o, \Delta, \sigma, b^2)$ with $\vecn = (\vecn_1, \dots, \vecn_{d_1})$, etc. the \emph{effective tuple for the neurons $1, \dots d_1$} and write $\cE_{d_1}$ for the set of all effective tuples using $d_1$ neurons.

 First we note that the response of a network can be represented in terms of the previous quantities:
\begin{align*}
\mathfrak N^\IW(x)&=b^2+\sum_{j=1}^{d_1} w^2_j(w^1_j \cdot x+b_j^1)_+=b^2+\sum_{j=1}^{d_1}\Delta_j \Bigl(\frac 1{|w_j^1|} w^1_j \cdot x+\frac 1{|w_j^1|} b_j^1\Bigr)_+\\
&=b^2+\sum_{j=1}^{d_1} \Delta_j \bigl(\sigma _j ( \vecn_j \cdot x- o_j)\bigr)_+,
\end{align*}
where $\cdot $ denotes the scalar product.  
Conversely, for arbitrary  positive normals $\vecn_1,\dots,\vecn_{d_1}\in \cO_+^1$, offsets $o_1,\dots,o_{d_1}\in\R$, kinks $\Delta_1,\dots,\Delta_{d_1}\in \R\backslash \{0\}$, orientations $\sigma_1,\dots,\sigma_{d_1}\in \{\pm1\}$ and $b^2\in\R$ the following mapping is a $C^\infty$-bijection onto all networks $\IW$ having the latter features:
\begin{align} \label{eq:manifold}
	\psi: (0, \infty)^{d_1} \to \mathcal W^{d_0,d_1,1}\; ; \; a\mapsto (\tilde W^1, \tilde b^1, \tilde W^2,b^2),
\end{align}
where 
$$
	\tilde W^1 = (\sigma_1 a_1 \vecn_1, \dots, \sigma_{d_1} a_{d_1} \vecn_{d_1})^\dagger \; , \; \tilde b^1=(-\sigma_1 a_1 o_1, \dots, -\sigma_{d_1} a_{d_1} o_{d_1})
$$
and
$$
	\tilde W^2 = \Bigl( \frac{\Delta_1}{a_1}, \dots, \frac{\Delta_{d_1}}{a_{d_1}}\Bigr).
$$
Thus, for $d_1 \in \N$ and an effective tuple $(\vecn, o, \Delta, \sigma, b^2)$ the set of neural networks $\IW \in \cW^{d_0,d_1,1}$ having the effective tuple $(\vecn, o, \Delta, \sigma, b^2)$ forms a $d_1$-dimensional $C^\infty$-manifold with one connected component.

With slight misuse of notation we write
$$
\mathfrak{N}^{\vecn, o, \Delta,\sigma,b^2}:\R^{d_0}\to \R,\, x\mapsto b^2+\sum_{j=1}^{d_1} \Delta_j \bigl(\sigma _j ( \vecn_j \cdot x- o_j)\bigr)_+
$$
and briefly write $\mathfrak{N}^{\vecn,o,\Delta,\sigma}=\mathfrak{N}^{\vecn,o,\Delta,\sigma,0}$ and $\mathfrak{N}^{\vecn,o,\Delta}=\mathfrak{N}^{\vecn,o,\Delta,(1,\dots,1)}$.  Although, the tuple $(\vecn,\Delta,\sigma,o,b^2)$ does not  uniquely describe a neural network, it describes a response function uniquely and thus we will speak of the neural network with \emph{effective} tuple  $(\vecn,o,\Delta,\sigma,b^2)$.
We note that each summand of  $\mathfrak{N}^{\vecn,o,\Delta,\sigma,b^2}$ has the breakline
$$
P_j=\bigl\{x\in\R^{d_0}:\vecn_j \cdot x =  o_j\bigr\}
$$
and we call 
$$
 A_j=\{x: \sigma_j(\vecn_j \cdot x-o_j)>0\}
$$
the domain of activity of the $j$th neuron. With slight misuse of notation we also call $(\vecn_j,o_j)$ \emph{breakline} of the $j$th neuron. Note that in the case where $\sigma_j=1$, $A_j$ is the part on the positive side of  the hyperplane $P_j$ and in the case where $\sigma_j=-1$, $A_j$ is the negative side of the hyperplane.
By construction, we have 
$$
\mathfrak{N}^{\vecn,o,\Delta,\sigma,b^2}(x)= b^2+\sum_{j=1}^{d_1} \1_{A_j}(x)  \bigl(\sigma _j \Delta_j( \vecn_j \cdot x- o_j)\bigr).
$$
Outside the breaklines the function $\mathfrak{N}^{\vecn,o,\Delta,\sigma,b^2}$ is differentiable with
$$
D\mathfrak{N}^{\vecn,o,\Delta,\sigma,b^2}(x)= \sum_{j=1}^{d_1} \1_{A_j}(x)  \sigma _j \Delta_j  \vecn_j.
$$
Note that for each summand  along a breakline the difference of the differential on the positive  and negative side equals $\Delta_j \vecn_j$.

\section{Minimal representations of shallow ReLU nets}
In this section, we fix a representable function $f \in \cF$ and derive the necessary number of neurons in the hidden layer in order to find a neural network having response $f$. Moreover, we classify the set of minimal networks with realization function $f$.
We start with the following auxiliary result on representable functions.

\begin{lemma}\label{le:std_rep}
	Let $f \in \mathcal F = \bigcup_{d_1\in\N}\{\mathfrak N^\IW : \IW \in \cW^{d_0,d_1,1}\}$. 
	\begin{enumerate} \item[(i)]  The function $f$ is affine or there exist  (up to reordering) unique pairwise different breaklines $S_1=(\vecn_1,o_1),\dots,S_n=(\vecn_n,o_n)$ $(n\in\N)$ and kinks $\Delta_1,\dots,\Delta_n\in\R\backslash \{0\}$ in the sense that there exist unique $a\in\R^{d_0}$ and $b\in\R$ with 
	$$
	f(x)=\mathfrak N^{(\vecn,o,\Delta)}(x) + a\cdot x+b,\text{ \ for all }x\in\R^{d_0}.
	$$
	\item[(ii)]  If $f$ is non-affine,  then for every $(\sigma_j)\in \{\pm1\}^n$, one has
	\begin{align}\label{rep1}
	f(x)=\mathfrak N^{(\vecn,o,\Delta,\sigma)}(x) + a_\sigma \cdot x+b_\sigma,
	\end{align}
	where
	$$
	a_\sigma =  a+\sum_{j:\sigma_j=-1} \Delta_j \vecn_j   \text{ \ and \ } b_\sigma =b-\sum_{j:\sigma_j=-1} \Delta_j o_j.
	$$
	\item[(iii)] 
	Let $d_1\in\N$ and  $\IW \in \cW^{d_0,d_1,1}$ be a network with breaklines $(\tilde \vecn_1,\tilde o_1),\ldots,(\tilde\vecn_{d_1},\tilde o_{d_1})$ and kinks $\tilde \Delta_1,\dots,\tilde \Delta_{d_1}\in\R$.
	The response  $\mathfrak N^\IW$  agrees with $f$ up to an affine function
	 iff for every breakline $S=(\vecm,o)$ of $f$ or $\IW$ 
	\begin{align}\label{eq93821}
	\sum_{j:(\tilde \vecn_j,\tilde o_j)=S} \tilde \Delta _j =\begin{cases} \Delta_j,& \text{ if $S$ is a breakline of $f$},\\
	0,&\text{ else.}\end{cases}
	\end{align}
	In particular, a representation  $\IW$ for the function $f$ possesses for each breakline of $f$ at least one neuron with the respective breakline.
	\end{enumerate}
\end{lemma}

\begin{proof} 
(i): Let $d_1 \in \N$ and $\IW \in \cW^{d_0,d_1,1}$  with $\mathfrak N^\IW = f$.  We  denote by $(\tilde\vecn,\tilde o,\tilde \Delta,\tilde \sigma,\tilde b^2)$ the effective tuple of $\IW$ and let  $\cP=\{(\tilde\vecn_1,\tilde o_1),\dots,( \tilde\vecn_{d_1},\tilde o_{d_1})\}$  denote the breaklines of the neurons.
For a breakline $S\in \cP$ we call
$$
\Delta(S)= \sum_{j: (\tilde \vecn_j,o_j)=S} \tilde \Delta_j
$$
the effective kink. We note that for  $S=(\vecm,o)\in\cP$, the function
	\begin{align*}
		g_{S}(x)&:=\sum_{j:(\tilde \vecn_j,\tilde o_j)=S} \1_{\{\tilde \sigma_j (\vecm \cdot  x-o)>0\}}(\tilde \sigma_j \tilde\Delta_j (\tilde \vecn_j \cdot  x-o_j)) - \1_{\{\vecm \cdot x - o>0\}} \Delta(S) (\vecm \cdot x-o) \\
		&= -\sum_{j:(\tilde \vecn_j,\tilde o_j)=S}\1_{\{-1\}}(\tilde \sigma_j) \tilde \Delta_j ( \vecm \cdot x - o)
	\end{align*}
is affine. We enumerate the breaklines of $\cP$ that have non-zero effective kinks, say by $(\vecn_1,o_1),\dots,(\vecn_n,o_n)$ and note that for $\vecn=(\vecn_1,\dots,\vecn_n)$, $o=(o_1,\dots,o_n)$ and $\Delta=(\Delta(\vecn_1,o_1),\dots,\Delta(\vecn_n,o_n))$, one has
\begin{align*}
f(x)&=\mathfrak N^\IW(x)= b^2+ \sum_{S\in \cP} \sum_{j:(\tilde \vecn_j,\tilde o_j)=S} \1_{\{\tilde \sigma_j (\vecm \cdot  x-o)>0\}}(\tilde \sigma_j \tilde\Delta_j (\tilde \vecn_j \cdot  x-o_j))\\
&=  \sum_{\ell=1}^n \1_{\{\vecn_\ell \cdot  x - o_\ell>0\}} \Delta_\ell (\vecn_\ell \cdot  x-o_\ell) + \sum_{S\in\cP} g_S(x) +b^2=\mathfrak N^{(\vecn,o,\Delta)}(x)+ \sum_{S\in\cP} g_S(x) +b^2.
\end{align*}
Consequently, $f-\cN^{(\vecn,o,\Delta)}$ is affine.
	
	To show uniqueness note that as all kinks $\Delta_1, \dots, \Delta_n$ are non-zero we have
	$$
		\{x \in \R^{d_0}: f(x) \text{ is not differentiable in }x\} = \bigcup_{i=1}^n \{x \in \R^{d_0}: \vecn_i \cdot  x =o_i\} 
	$$
	and $\Delta_i \vecn_i$ is the difference of the differential on the positive and negative side of $\{x \in \R^{d_0}: \vecn_i \cdot x =o_i\}$,
	so that the breaklines $S_1=(\vecn_1, o_1), \dots, S_n= (\vecn_n,o_n)$ and the kinks $\Delta_1, \dots, \Delta_n$ are unique up to reordering. Now let $x \in\R^{d_0}$ such that $f$ is differentiable in $x$. Then,
	$$
		Df(x)= D\mathfrak N^{(\vecn, o, \Delta)}(x)+a 
	$$
	so that $a$ and $b$ are uniquely determined by $(\vecn, o, \Delta)$ and the gradient and function value of $f$ at point $x$.

(ii): We show formula~(\ref{rep1}). Again, note that 
$$
\Delta_j(\vecn_j \cdot x-o_j)_+-\Delta_j(-\vecn_j \cdot x+o_j)_+=\Delta_j(\vecn_j \cdot x-o_j)
$$
so that
\begin{align*}
\mathfrak N^{(\vecn,o,\Delta)}(x)-\mathfrak N^{(\vecn,o,\Delta,\sigma)}(x) &=\sum_{j:\sigma_j=-1} (\Delta_j(\vecn_j\cdot  x-o_j)_+-\Delta_j(-\vecn_j \cdot x+o_j)_+)\\
&=\sum_{j:\sigma_j=-1}\Delta_j(\vecn_j \cdot x-o_j).
\end{align*}
The statement now follows with
$$
f(x)=\mathfrak N^{(\vecn,o,\Delta)}(x) + a\cdot x+b= \mathfrak N^{(\vecn,o,\Delta,\sigma)}(x) +\sum_{j:\sigma_j=-1}\Delta_j(\vecn_j\cdot x-o_j)+a\cdot x+b.
$$

(iii): As we showed in (i), for every network $ \IW'\in \cW^{d_0,d_1',1}$ with breaklines $\cP'$ one has that  the response $\mathfrak N^{\IW'}$ agrees up to an affine function with
$$
\R^d\ni x\mapsto \sum_{S=(\vecm,o)\in  \cP'} \1_{\{\vecm \cdot  x-o>0\}}  \Delta'(S) (\vecm \cdot x-o),
$$ 
where $\Delta'(S)$ denotes the effective kink of the breakline $S$ in the network $ \IW'$. The latter function (and thus the response $\mathfrak N^{\IW'}$) is affine if and only if all effective kinks of $\IW'$ are zero.

Now note that using (i), $f-\mathfrak N^{\IW}$ is the response of a network $\IW'$ with breaklines being the union of the breaklines of $f$ and $\IW$. Now this function is affine if and only if the effective kinks of all breaklines of $\IW'$ are  zero which agrees with  the validity of~(\ref{eq93821}) for every breakline of $f$ and $\IW$.
\end{proof}

Let $f\in\cF$ and denote by $S_1=(\vecn_1,o_1),\ldots, S_n=(\vecn_n,o_n)$ the unique (up to reordering) breaklines of $f$. We will use the following notation: for $\vecn,\vecn'\in \cN:=\{\vecn_1, \ldots,  \vecn_n\}$, let
$$
J= \{\sigma \in\{\pm 1\}^{n}: a_\sigma=0\}, \ \  J(\vecn)=\{\sigma\in\{\pm 1\}^{n}:  a_\sigma \in \R \vecn\}
$$
and
$$
J(\vecn,\vecn')=\{\sigma\in\{\pm 1\}^{n}:  a_\sigma \in \R \vecn+\R \vecn'\}.
$$
Moreover, we let 
$$
\II(\vecn)=\{j\in\{1,\dots,n\}:\vecn_j=\vecn\}.
$$
Note that for a $\sigma \in J(\vecn)$ and a $j \in \II(\vecn)$ we also have $(\sigma_1, \dots, \sigma_{j-1}, -\sigma_j, \sigma_{j+1}, \dots, \sigma_n) \in J(\vecn)$. Therefore, in that case one can always find a $\sigma \in J(\vecn)$ with $\sigma_j=1$. Analogously, for $\sigma \in J(\vecn, \vecn')$ and $j_1 \in \II(\vecn)$ and $j_2 \in \II(\vecn')$ one can switch the $j_1$th and $j_2$th orientation and still get a tuple of orientations in $J(\vecn, \vecn')$.

\begin{theorem} \label{thm:main}
Let $f\in\cF$ be non-affine and denote by $S_1=(\vecn_1,o_1),\ldots, S_n=(\vecn_n,o_n)$ the unique (up to reordering) breaklines of $f$ and by $\Delta_1, \dots, \Delta_n$ the respective kinks. 
\begin{enumerate}
\item[(i)] If the set $J$ is not empty, then the minimal representation for $f$ uses $n$ neurons and there is a bijection between the  set of all effective tuples $\cE_n$  with response $f$ modulo permutation and the set
$$
J.
$$
\item[(ii)] If the set $J$ is empty and for a $\vecm\in\mathcal N:=\{\vecn_1,\dots,\vecn_n\}$ the set
$
J(\vecm)$
is not empty, then the minimal representation for $f$ uses $n+1$ neurons and there is a bijection between the set of effective tuples $\cE_{n+1}$ with response $f$ modulo permutation and
$$
 \bigcup_{\vecn\in \cN: J(\vecn)\not =\emptyset} \bigl\{(\sigma,j)\in  J(\vecn)\times \II(\vecn): \sigma_j=1\bigr\}.
$$
\item[(iii)] If for every $\vecn\in\cN$, $J(\vecn)=\emptyset$, then the minimal representation for $f$ uses $n+2$ neurons and there is bijection between the set of effective tuples $\cE_{n+2}$ with response $f$ modulo permutation and
$$
 \{\pm1\}^n\times \R \cup\bigcup_{\vecn,\vecn'\in  \cN:J(\vecn,\vecn')\not= \emptyset, \vecn<\vecn'} \bigl\{(\sigma,j_1,j_2)\in  J(\vecn,\vecn')\times \II(\vecn)\times \II(\vecn') : \sigma_{j_1}=\sigma_{j_2}=1\bigr\},
$$
where ``$<$'' refers to an arbitrarily fixed ordering of $\cN$.
\end{enumerate}
\end{theorem}

\begin{proof}We characterize the set of effective tuples modulo permutation. We enumerate the breaklines $(\vecn_1,o_1),\dots,(\vecn_n, o_n)$ and denote by   $\Delta_1,\dots,\Delta_n\in\R\backslash \{0\}$ the respective kinks of $f$.


(i): In view of Lemma~\ref{le:std_rep}, we have for $\sigma\in\{\pm1\}^n$ and $b^2\in\R$ that $f=\mathfrak N^{(\vecn, o, \Delta,\sigma, b^2)}$ if and only if
$$
a_\sigma=0 \text{ \ and \ } b^2=b_\sigma.
$$
Hence, if $J$ is not empty, then there exists a representation for $f$ with $n$ neurons and we get for every $\sigma\in J$, that  $(\vecn, o, \Delta,\sigma, b_\sigma)$ is a different representation for $f$. 

It remains to prove that up to the order of the neurons these are the only representations with $n$ neurons. By Lemma~\ref{le:std_rep}, $(\vecn,o,\Delta)$ is the unique tuple (up to reordering the breaklines) for which $f$ and $\mathfrak N^{\vecn,o, \Delta}$ coincide with $f$ up to an affine function. Hence, there are no other representations beyond the ones from above.

(ii): If $J=\emptyset$, then there is no representation for $f$ using $n$ neurons. Now let $\vecn\in\cN$ with $J(\vecn)\not =\emptyset$. Take $j\in \II(\vecn)$ and  $\sigma\in J(\vecn)$ with $\sigma_j=1$.
We will provide a representation for $f$ with  $n+1$ neurons depending on the choice of $\vecn$, $j\in \II(\vecn)$ and $\sigma\in J(\vecn)$. We choose for every $i=1,\dots,n$, $(\vecn'_{i},o_i',\sigma_i')=(\vecn_i,o_i,\sigma_i)$ and $(\vecn'_{n+1},o_{n+1}',\sigma_{n+1}')=(\vecn_j,o_j,-1)$ meaning that we added one neuron with the same breakline as the $j$th neuron with the $j$th neuron being positively and the $n+1$st negatively oriented. Moreover, we choose $ \Delta'\in \R^{n+1}$ with $\Delta'_i=\Delta_i$ for $i\in\{1,\dots, n\}\backslash\{j\}$.
Since $\sigma\in J(\vecn)$ we have for a $\delta _\sigma\in \R$ that $a_\sigma =\delta_\sigma \vecn_j$. 
Using that 
$$
f(x)= \mathfrak N^{(\vecn,o,\Delta,\sigma)}(x)+ \delta_\sigma \vecn_j \cdot x+b_\sigma
$$
and
$$
\mathfrak N^{(\vecn',o',\Delta',\sigma',b^2)}(x)  = \mathfrak N^{(\vecn,o,\Delta,\sigma)}+ \bigl(\1_{A_j}(x) \Delta_j'- \1_{A_j^c}(x) \Delta_{n+1}'- \1_{A_j}(x) \Delta _j\bigr)(\vecn_j \cdot x -o_j) +b^2
$$
we get that $f=\mathfrak N^{(\vecn',o',\Delta',\sigma',b^2)}$ if and only if $\Delta_j'$, $\Delta_{n+1}'$ and $b^2$ are such that
\begin{align}\label{eq94571}
\Delta_j'= \Delta_j+ \delta_\sigma, \ \Delta_{n+1}'=-\delta_\sigma \text{ \ and \ }  b^2=\delta_\sigma o_j +b_\sigma.
\end{align}
In particular, we showed that in this case there are representations with $n+1$ neurons that are minimal.
The above construction is injective in the sense that every tuple $(\vecn, j, \sigma)$ as above produces a different representation for $f$ (modulo reordering).
Indeed, this is the case, since whenever, for two tuples $(\vecn, j, \sigma)$ and $(\vecn', j', \sigma')$ as above, such that $j$ does not agree with $j'$ (or $\vecn$ does not agree with $\vecn'$), then the breakline that is shared by  two neurons is different for the two tuples and the representations are different even modulo reordering. Moreover, in the case where $j=j'$ and $\sigma\not=\sigma'$ there is a breakline that is associated with exactly one neuron in both representations, but in the one representation the neuron is positively and in the other representation the neuron is negatively oriented. 
Again the two representations disagree even when allowing reordering.

It remains to show that all representations are obtained by the above construction.
 By Lemma~\ref{le:std_rep} (iii), there exists for each breakline of $f$ at least one neuron with this breakline.  So after discarding one neuron for each breakline there will be one neuron left. By minimality this neuron has a kink and if it would induce a breakline that does not coincide with one of the other neurons, then $f$ would have $n+1$ breaklines which contradicts the assumption. So there has to be one breakline that is shared by two neurons. If the orientation of the two neurons  would agree one could reduce the number of neurons by discarding one of the neurons and taking for the remaining one the effective kink being the sum of the two kinks.
 By reordering we can assure that the neurons $1,\dots,n$ have  breaklines $(\vecn_1,o_1), \dots,(\vecn_n,o_n)$ and we denote by $j$ the index of the breakline that is shared by two neurons. We suppose that  the $j$th neuron is positively oriented and the $n+1$ neuron is negatively oriented (otherwise we interchange the neurons) and we denote by $(\vecn',o',\Delta',\sigma',b^2)$ the effective tuple of the respective network. 
We will show that the latter network is obtained in the above construction when starting with $\vecn_j$, $\sigma$ being the orientation of the first $n$ neurons and the index $j$ of the doubled breakline. 
 
By Lemma~\ref{le:std_rep} (iii), $\Delta_i=\Delta_i'$ for $i\in\{1,\dots,n\}\backslash \{j\}$ and $\Delta_j=\Delta_j'+\Delta_{n+1}'$. Hence, noting that
$$
f(x)= \mathfrak N^{(\vecn,o,\Delta,\sigma)}(x)+a_\sigma \cdot x+b_\sigma
$$
and
\begin{align*}
f(x)&=\mathfrak N^{(\vecn',o',\Delta',\sigma',b^2)}(x) \\
&= \mathfrak N^{(\vecn,o,\Delta,\sigma)}(x)+ \bigl(\1_{A_j}(x) \Delta_j'- \1_{A_j^c}(x) \Delta_{n+1}'- \1_{A_j}(x) \Delta _j\bigr)(\vecn_j \cdot  x -o_j)+b^2
\end{align*}
  we conclude that $a_\sigma=-\Delta'_{n+1} \vecn_j$ and $b^2=-\Delta_{n+1}'o_j+b_\sigma$. With $\delta_\sigma=-\Delta_{n+1}'$ we get validity of all equations in~(\ref{eq94571}). In particular,
 $$
 f(x)=\mathfrak N^{(\vecn,o,\Delta,\sigma)}(x)+\delta_\sigma \vecn_j \cdot  x+b_\sigma
 $$
 so that  $\sigma\in J(\vecn_j)$.

(iii): Now we suppose that for every $\vecm \in \cN$ the set $J(\vecm)$ is empty in which case there do not exist representations of $f$ with $n+1$ neurons.

We provide two procedures to construct representations with $n+2$ neurons for $f$. First let $\sigma\in\{\pm1\}^n$ and $r\in \R$ arbitrary. We consider a network 
$(\vecn', o',\Delta', \sigma',b^2)$ with $n+2$ neurons that is obtained as follows. We let
for $i=1,\dots,n$, $(\vecn'_i, o'_i,\Delta'_i,\sigma'_i)=(\vecn_i,o_i,\Delta_i,\sigma_i)$ and  $$(\vecn_{n+1}',o'_{n+1})=(\vecn_{n+2}', o'_{n+2})= \Bigl(\frac {\orient (a_\sigma)}{| a_\sigma|}\, a_\sigma, r\Bigr),\, \sigma_{n+1}'=1, \, \sigma_{n+2}'=-1.$$
We note that
$$
f(x)= \mathfrak N^{(\vecn,o,\Delta,\sigma)}(x) +a_\sigma \cdot  x+b_\sigma
$$
and
$$
\mathfrak N^{(\vecn',o',\Delta',\sigma',b^2)}(x)= \mathfrak N^{(\vecn,o,\Delta,\sigma)}(x) +\Delta_{n+1}' (\vecn_{n+1}' \cdot x-o_{n+1}')_++\Delta_{n+2}' (-(\vecn_{n+1}' \cdot x-o_{n+1}'))_++b^2.
$$
Consequently, $\mathfrak N^{(\vecn',o',\Delta',\sigma',b^2)}=f$ iff $\Delta'_{n+1}$, $\Delta'_{n+2}$ and $b^2$ satisfy
$$
\Delta'_{n+1}=-\Delta_{n+2}'= \orient(a_\sigma) \,|a_\sigma| \text{ and } b^2=  \orient(a_\sigma) \,|a_\sigma| \,r+b_\sigma.
$$
Thus, we constructed for every given $\sigma\in\{\pm1\}^n$ and $r\in \R$ an effective  neural network that represents $f$. Note that for any choice of the parameters the first $n$ neurons have the breaklines of the function $f$ with orientations being given by $\sigma$ and the $n+1$st and $n+2$th neuron have identical breaklines that are not in $\cN$; the former one being positively oriented and the latter one being negatively oriented. This entails that two different tuples $(\sigma,r)$ and $(\tilde \sigma,\tilde r)$ yield different representations even modulo reordering of the neurons. Indeed, if $\sigma\not=\tilde \sigma$, then there is a breakline of $f$ that appears  in both representations once but with opposite orientations. Thus the effective representations are different. Moreover, if $\sigma=\tilde \sigma$ but $r\not =\tilde r$, then both representations have exactly two neurons with identical breaklines but the respective breaklines disagree since they have different offsets $r$ and $\tilde r$.

 We proceed with a second procedure to construct further representations that are not covered by the first construction. Let $\vecm,\vecm'\in \cN$ with $\vecm<\vecm'$ and $J(\vecm,\vecm')\not =\emptyset$. We choose $j_1\in \II(\vecm)$ and $j_2\in \II(\vecm')$ and let $\sigma \in J(\vecm, \vecm')$ with $\sigma_{j_1}=\sigma_{j_2}=1$. We define an effective network $(\vecn',o',\Delta',\sigma',b^2)$ that represents $f$ by letting
 for $i=1,\dots,n$, $(\vecn_i',o_i',\sigma_i')=(\vecn_i,o_i,\sigma_i)$ and, for $l\in\{1,2\}$,
 $$
 (\vecn_{n+l}',o_{n+l}',\sigma_{n+l}')=(\vecn_{j_l},o_{j_l},-1).
$$
This  means that we add two neurons with the same breaklines as the $j_1$th and the $j_2$nd  neuron but with negative (opposite) orientation. Moreover, we let for $i\in\{1,\dots,n \} \backslash \{j_1,j_2\}$, $\Delta_i'=\Delta_i$.
Using that
\begin{align*}
\mathfrak N^{(\vecn',o',\Delta',\sigma',b^2)}(x)&= \mathfrak N^{(\vecn,o,\Delta,\sigma)}(x)+ \bigl(\1_{A_{j_1}}(x) \Delta_{j_1}'- \1_{A_{j_1}^c}(x) \Delta_{n+1}'- \1_{A_{j_1}}(x) \Delta _{j_1}\bigr)(\vecn_{j_1} \cdot x -o_{j_1})\\
&\qquad + \bigl(\1_{A_{j_2}}(x) \Delta_{j_2}'- \1_{A_{j_2}^c}(x) \Delta_{n+2}'- \1_{A_{j_2}}(x) \Delta _{j_2}\bigr)(\vecn_{j_2} \cdot  x -o_{j_2})+b^2
\end{align*}
and representation~(\ref{rep1}) we get that $\mathfrak N^{(\vecn',o',\Delta',\sigma',b^2)}=f$ iff
\begin{align*}
a_\sigma x+b_\sigma& =\bigl(\1_{A_{j_1}}(x) \Delta_{j_1}'- \1_{A_{j_1}^c}(x) \Delta_{n+1}'- \1_{A_{j_1}}(x) \Delta _{j_1}\bigr)(\vecn_{j_1} \cdot x -o_{j_1})\\
&\qquad + \bigl(\1_{A_{j_2}}(x) \Delta_{j_2}'- \1_{A_{j_2}^c}(x) \Delta_{n+2}'- \1_{A_{j_2}}(x) \Delta _{j_2}\bigr)(\vecn_{j_2} \cdot  x -o_{j_2})+b^2.
\end{align*}
Recalling that by assumption $J(\vecm,\vecm')\not =\emptyset$ and that $\vecm,\vecm'$ are linearly independent we conclude that there exist unique values $\delta_ {1},\delta_ {2}\in \R$ with
$$
a_\sigma=\delta_{1} \vecn_{j_1}  +\delta_{2} \vecn_{j_2}.
$$
We conclude that $\mathfrak N^{(\vecn',o',\Delta',\sigma',b^2)}=f$ iff for $l \in \{1,2\}$ 
\begin{align}\label{eq945712}
\Delta_{j_l}'= \Delta_{j_l}+ \delta_l, \ \Delta_{n+l}'=-\delta_l \text{ \ and \ }  b^2=\delta_1 o_{j_1}+ \delta_2 o_{j_2} +b_\sigma.
\end{align}
The above construction is injective in the sense that every tuple $(\vecm, \vecm', j_1, j_2, \sigma)$ as above produces a different representation for $f$ (modulo reordering).
Indeed, let $(\vecm, \vecm', j_1, j_2, \sigma)$ and $(\tilde\vecm, \tilde\vecm', \tilde j_1, \tilde j_2, \tilde \sigma)$ be two different tuples as above. Note that the two added neurons $n+1$ and $n+2$ both are negatively oriented. So if $\sigma\not =\tilde \sigma$, there is a breakline of $f$ that is addressed by a positively oriented neuron in one representation but only with negatively oriented neurons in the other representation. Moreover, if $(\vecm, \vecm')\not =( \tilde\vecm, \tilde\vecm')$, then $(j_1, j_2)\not =(\tilde j_1,\tilde j_2)$ and generally if $(j_1, j_2)\not =(\tilde j_1,\tilde j_2)$, then the number of neurons that correspond to the $j_1$th or $j_2$nd breakline of~$f$ disagree and again the effective representations disagree even when allowing reordering of the neurons. 

Also, effective networks that are created by the second construction have two breaklines sharing two neurons whereas effective networks created by the former construction only have one breakline sharing two neurons, so that the set of effective networks created by the two different constructions are disjoint.

It remains to show that all representations are obtained by the above two constructions. Let $(\vecn',o',\Delta',\sigma',b^2)$ be a minimal  effective representation of $f$ with $n+2$ neurons.

By Lemma~\ref{le:std_rep} (iii), there exists for each breakline of $f$ at least one neuron with this breakline.  So after discarding one neuron for each breakline there will be two neurons left. By minimality these neurons have a non-vanishing kink and a breakline. Suppose that one of the remaining two neurons has a breakline, say $(\vecm,r)$ that is not in $\cN$. If the breakline would not agree with the one of the second remaining neuron, then the effective representation describes a function that  is not differentiable on the breakline $(\vecm,r)$ and thus is not equal to $f$. In that case the breaklines of the two remaining neurons would be identical.
 Also, by minimality each breakline has at most for every orientation one neuron  so that one of the following cases prevails:
\begin{enumerate}
\item The two remaining neurons have identical breaklines that are not in $\cN$. 
	\item There exist $1\le j_1 < j_2 \le n$ such that there exist two neurons with breaklines  $(\vecn_{j_1},o_{j_1})$, resp. $(\vecn_{j_2},o_{j_2})$ (one of them having positive orientation and the other having negative orientation).
\end{enumerate}

(1): By reordering we can assure that the neurons $1,\dots,n$ have  breaklines $(\vecn_1,o_1), \dots,(\vecn_n,o_n)$, and the $n+1$ and $n+2$ have the breakline $(\vecm, r)$ with the former neuron being positively oriented and the latter one being negatively oriented. 
Using that $(\vecm,r)\notin \cN$ 
we get with Lemma~\ref{le:std_rep} (iii) that $$\Delta_{n+1}'=-\Delta_{n+2}'\text{ \ and \ } \Delta_k'=\Delta_k, \text{ for }k=1,\dots,n.$$ 
We set $\sigma_k=\sigma_k'$ for $k=1,\dots,n$ and recall that
\begin{align*}
	f(x)&= \mathfrak N^{(\vecn,o,\Delta,\sigma)}(x) +a_\sigma \cdot  x+b_\sigma.
\end{align*}
Again we use that
$$
 \mathfrak N^{(\vecn',o',\Delta',\sigma',b^2)}(x) = \mathfrak N^{(\vecn,o,\Delta,\sigma)}(x) +\Delta_{n+1}'(\vecm \cdot  x-r)+b^2
$$
to conclude that $ \mathfrak N^{(\vecn',o',\Delta',\sigma',b^2)}=f$ iff
$$
	a_\sigma \cdot  x+b_\sigma = \Delta_{n+1}'(\vecm \cdot  x- r)+b^2
$$
or, equivalently,
$$
	\vecm = \frac{\orient(a_\sigma)}{|a_\sigma|}a_\sigma, \ \Delta_{n+1}'=-\Delta_{n+2}'=\orient(a_\sigma)|a_\sigma| 
$$
and $b^2 = b_\sigma+\Delta_{n+1}' r$. Thus the effective representation is obtained by the first construction when choosing the respective orientation $\sigma$ and offset $r$. 

(2): By reordering we can assure that the neurons $1,\dots,n+2$ have  breaklines $(\vecn_1,o_1), \dots,(\vecn_n,o_n),$ $(\vecn_{j_1},o_{j_1}),(\vecn_{j_2},o_{j_2})$, where the $j_1$th and $j_2$th neuron are positively  and the $n+1$th and $n+2$th neuron are negatively oriented and we set $\sigma_k=\sigma_k'$ for $k=1, \dots, n$.

By Lemma~\ref{le:std_rep} (iii), $\Delta_k=\Delta_k'$ for $k\in\{1,\dots,n\}\backslash \{j_1, j_2\}$ and, for $l \in \{1,2\}$, $\Delta_{j_l}=\Delta_{j_l}'+\Delta_{n+l}'$. Hence, noting that
$$
f(x)= \mathfrak N^{(\vecn,o,\Delta,\sigma)}(x)+a_\sigma \cdot  x+b_\sigma
$$
and
\begin{align*}
f(x)&=\mathfrak N^{(\vecn',o',\Delta',\sigma',b^2)}(x)\\
&= \mathfrak N^{(\vecn,o,\Delta,\sigma)}(x)+ \bigl(\1_{A_{j_1}}(x) \Delta_{j_1}'- \1_{A_{j_1}^c}(x) \Delta_{n+1}'- \1_{A_{j_1}}(x) \Delta _j\bigr)(\vecn_j \cdot  x -o_j)\\
&\ \ \ + \bigl(\1_{A_{j_2}}(x) \Delta_{j_2}'- \1_{A_{j_2}^c}(x) \Delta_{n+2}'- \1_{A_{j_2}}(x) \Delta _j\bigr)(\vecn_j \cdot  x -o_j)+b^2
\end{align*}
we conclude that for  $\delta_{l}:=-\Delta_{n+l}'$ ($l=1,2$),
$$ \Delta_{j_l}'=\Delta_{j_l}+\delta_l, \ \Delta'_{n+l}=-\delta_l, \ a_\sigma=\delta_1  \vecn_{j_1}+\delta_2\vecn_{j_2}\text{ \ and \ }b^2=\delta_{1}o_{j_1}+\delta_{2}o_{j_2}+b_\sigma$$
so that (\ref{eq945712}) is satisfied and, in particular, $\sigma\in J(\vecn_{j_1},\vecn_{j_2})$. Hence, the representation is obtained when following the second construction with $\vecm=\vecn_{j_1}$, $\vecm'=\vecn_{j_2}$ and $\sigma$, $j_1$ and $j_2$ as above. 
\end{proof}

\begin{rem}[One dimensional input ($d_0=1$)]
	Suppose that $d_0=1$ and choose $\cO_+=(0,\infty)$  and $\cO_-=(-\infty,0)$. Note that in this case there is a unique positive normal namely $\vecm = 1$. Hence, in that case effective networks with $d_1$ neurons in the hidden layer are parameterized by a tuple $$(o, \Delta,\sigma, b^2)\in \R^{d_1}\times \R^{d_1} \times \{\pm1\}^{d_1}\times \R.$$
	Now let $f\in\cF$ having $n$ breaklines. We note that since  $J(1)=\{\pm1\}^n\not=\emptyset$, case three of Theorem~\ref{thm:main} cannot enter. Therefore, we have the following dichotomy:
	\begin{enumerate}
		\item[(i)] If the set $J$ is not empty, then the minimal representation for $f$ uses $n$ neurons and there is a bijection between the  set of all effective tuples $\cE_n$  with response $f$ modulo permutation and the set
		$$
		J.
		$$
		\item[(ii)] If the set $J$ is empty, then the minimal representation for $f$ uses $n+1$ neurons and there is a bijection between the set of effective tuples $\cE_{n+1}$ with response $f$ modulo permutation and
		$$
		\{(\sigma,j) \in\{\pm1\}^{n}\times\{1,\dots,n\}: \sigma_j=1\}\simeq \{\pm 1\}^{n-1} \times \{1, \dots, n\}.
		$$
	\end{enumerate}
\end{rem}

\begin{rem} (Affine realization function)  Let $f(x)=a\cdot x+b$ for $a\in \R^{d_0} \backslash\{0\}$ and $b \in \R$. Then, analogously to the case (iii) in Theorem~\ref{thm:main}, one can show that a minimal representation for $f$ uses $2$ neurons. Further
	\begin{align} \label{eq:manifold2}
		\phi:(0,\infty)^2 \times \R \to \cW^{d_0,d_1,1} \; ; \; r \mapsto(\tilde W^1, \tilde b^1,\tilde W^2, \tilde b^2),
	\end{align}
	where
	$$
		\tilde W^1=\Bigl(\orient(a)r_1\frac{a}{|a|}, -\orient(a)r_2\frac{a}{|a|}\Bigr)^\dagger \; , \; \tilde b^1= (-\orient(a)r_1r_3, \orient(a)r_2r_3),
	$$
	$$
		\tilde W^2 = \Bigl( \frac{|a|}{r_1}, \frac{|a|}{r_2} \Bigr) \; \text{ and } \; \tilde b^2=b+|a|r_3	
	$$
	is a $C^\infty$-diffeomorphism onto the networks having two neurons (the first one positively orientated and the second one negatively) with response $f$. Therefore, the set of networks $\IW \in \cW^{d_0,2,1}$ with response $f$ forms a $3$-dimensional $C^\infty$-manifold having $2$ connected components.
\end{rem}

\begin{rem} (Manifold of minimal representations) Using the mapping (\ref{eq:manifold}) and (\ref{eq:manifold2}) we can characterize the manifold of minimal representations for a given function $f:\R^{d_0} \to \R$ having $n$ breaklines. 
	\begin{enumerate}
		\item[(i)] If the set $J$ is not empty, then the set of all networks $\IW \in \cW^{d_0,n,1}$ with response $f$ forms a $n$-dimensional $C^\infty$-manifold having $n! \cdot |J|$ connected components.
		\item[(ii)] If the set $J$ is empty and for a $\vecm \in \cN$ the set $J(\vecm)$ is not empty, then the set of all networks $\IW \in \cW^{d_0,n+1,1}$ with response $f$ forms a $n+1$-dimensional $C^\infty$-manifold having $$
			(n+1)! \cdot \Bigl( \sum\limits_{n \in \cN:J(\vecn)\neq \emptyset} \bigl| \bigl\{(\sigma,j)\in  J(\vecn)\times \II(\vecn): \sigma_j=1\bigr\} \bigr| \Bigr)
		$$
		connected components.
		\item[(iii)] If for every $\vecn, \vecn'\in\cN$, $J(\vecn,\vecn')=\emptyset$, then the set of all networks $\IW \in \cW^{d_0,n+2,1}$ with response $f$ forms a $n+3$-dimensional $C^\infty$-manifold having $(n+2)! 2^n$ connected components. 
		
		If, however, for every $\vecn \in \cN$, $J(\vecn)=\emptyset$, but there exist $\vecn, \vecn'\in\cN$ with  $J(\vecn,\vecn')\neq\emptyset$ then the set of all networks $\IW \in \cW^{d_0,n+2,1}$ with response $f$ forms a disjoint union of two different $C^\infty$-manifolds: One of dimension $n+3$ having $(n+2)! \cdot  2^n$ connected components and one of dimension $n+2$ having
		$$
			(n+2)! \cdot \Bigl(  \sum\limits_{\vecn,\vecn'\in  \cN:J(\vecn,\vecn')\not= \emptyset, \vecn<\vecn'} \bigl| \bigl\{(\sigma,j_1,j_2)\in  J(\vecn,\vecn')\times \II(\vecn)\times \II(\vecn') : \sigma_{j_1}=\sigma_{j_2}=1\bigr\} \bigr|\Bigr)
		$$
		connected components.
	\end{enumerate}
\end{rem}

\section{Representability of piecewise affine functions via shallow ReLU networks}

In this section, we discuss which continuous, piecewise affine functions
 $f:\R^{d_0} \to \R$ can be represented as response of a neural network. 

\begin{defi}
	Let $f:\R^{d_0} \to \R$ be a continuous function, $n \in \N$ and $S_1, \dots, S_n$ be  distinct hyperplanes  in $\R^{d_0}$. We call $f$ \emph{piecewise affine with breaklines $S_1, \dots, S_n$} if $f$ is 
	affine on each connected component of $\R^{d_0}\setminus (S_1\cup \ldots\cup S_n)$.
	We call the connected components of  $\R^{d_0}\setminus (S_1\cup \ldots\cup S_n)$ \emph{cells} of the breaklines $S_1,\dots,S_n$.
\end{defi}

As a first result, we note that for  $d_0>1$ there are  hyperplanes $S_1,\dots,S_n$ and piecewise affine functions $f$ with breaklines $S_1,\dots,S_n$ that cannot be represented as realization function of a shallow network.
\begin{exa}[A piecewise affine function that is not representable by a shallow network]
	Consider the piecewise affine function
	$$
	f(x,y) =  \begin{cases} 0, & \text{ if } |x|\ge |y| \text{ and } \sign(x)=-\sign(y) \text{ or } y=0,\\
	 y, & \text{ if } \sign(x)=\sign(y), \\ 
	 x+y, & \text{ else.}\end{cases} 
	$$
	Assume that there exist $d_1 \in \N$ and $\IW \in \cW^{2,d_1,1}$ with $\mathfrak N^\IW=f$.
	As consequence, see (\ref{eq:response}), the function  
	$$
	\R \setminus \{0\} \ni x \mapsto \lim\limits_{y \searrow 0} \nabla f(x,y)- \lim\limits_{y \nearrow 0} \nabla f(x,y)
	$$
	needs to be constant. But
	$$
	\lim\limits_{y \searrow 0} \nabla f(-1,y)- \lim\limits_{y \nearrow 0} \nabla f(-1,y) =\begin{pmatrix} 0 \\ -1 \end{pmatrix} \neq \begin{pmatrix} 0 \\ 1 \end{pmatrix} =  \lim\limits_{y \searrow 0} \nabla f(1,y)- \lim\limits_{y \nearrow 0} \nabla f(1,y)
	$$
	yields a contradiction and $f$ cannot be represented as response of a network.
\end{exa}

We note that in the example the function has  three breaklines that all meet in zero. As the following lemma shows  a function is representable, if all breaklines having an arbitrarily fixed point in common   have linear independent normals.  
Obviously, this property does not hold in the counterexample. In particular, we note that for the next lemma to hold one at least needs that there are at most $d_0$ breaklines that have an arbitrarily fixed point in common.

\begin{theorem} \label{lem:dim}
	Let $f:\R^{d_0} \to \R$ be a continuous and piecewise affine  function with breaklines $S_1, \dots S_n$. 
	If
	for all $x \in \R^{d_0}$ with $\# \rho(x)\neq 0$ we have 
	\begin{align} \label{eq:assurep}
	\dim\Bigl( \bigcap\limits_{j \in \rho(x)} S_j \Bigr) =d_0-\# \rho(x),
	\end{align}
	where $\rho(x):=\{j \in \{1, \dots, n\}: x \in S_j\}$, 
	then there exists a network $\IW \in \cW^{d_0,n+2,1}$ with 
	$$
	\mathfrak N^\IW = f.
	$$ 
\end{theorem}

\begin{rem}
Let $S_1,\dots,S_n$ be hyperplanes in $\R^{d_0}$. We note that the hyperplanes $S_1,\dots ,S_n$ do not satisfy assumption~(\ref{eq:assurep}) of the latter theorem iff there are indices $1\le j_1<\ldots<j_k\le n$ with $\emptyset\not=  S_{j_1}\cap\ldots \cap S_{j_{k-1}}\subset S_{j_k}$.

We note that in the case where $S_1,\dots,S_n$ denote the breaklines  of the hidden neurons of a network $\IW \in \cW^{d_0,n,1}$ we get that for Lebesgue almost-every choice of the network $\IW$ the breaklines satisfy  assumption~(\ref{eq:assurep}) of the latter theorem. (Indeed, one even has that for Lebesgue almost-every $\IW$, the function $\#\rho$ is bounded by two.)

%
	Recalling that the the mapping $\IW \mapsto \mathfrak N^{\IW}$ maps continuously into the space of continuous functions on a compact set endowed with uniform convergence we conclude that the universal approximation property (\cite{LESHNO1993861}, \cite{HORNIK19931069}) remains true when restricting to shallow networks for which  the breaklines of the hidden neurons satisfy assumption~(\ref{eq:assurep}).
\end{rem}

\begin{proof}
	Let $f:\R^{d_0} \to \R$ be a piecewise affine function with breaklines $S_1, \dots, S_n$ satisfying (\ref{eq:assurep}). Again we use the standard representation for the hyperplanes $S_1,\dots,S_n$ and denote by $\vecn_1,\dots,\vecn_n$ positively oriented normal vectors and by $o_1,\dots,o_n\in\R$ offsets such that for every  $k=1, \dots, n$ 
	$$
	S_k = \{x \in \R^{d_0}:  \vecn_k \cdot x  = o_k\}.
	$$
	Further, for $\sigma=(\sigma_i)_{1\le i\le n} \in \{-1,1\}^n$ we let
	$$
	\mathcal C(\sigma)= \{x \in \R^{d_0} : \sigma_k  (\vecn_k \cdot  x ) > \sigma_k o_k\text{ for all }k=1,\dots,n\}.
	$$
	Note that whenever $\mathcal C(\sigma)$ is not empty it is a cell. Moreover, every cell admits such a representation. 
	On a cell $\cC(\sigma)$ the gradient of $f$ is constant and we denote this constant by $\grad f_\sigma$.
%
	
	First, we show that for a cell $\cC(\sigma)$,  $x \in \overline{\cC(\sigma)}$ and $\bar \sigma\in\{\pm1\}^n$ with $\bar \sigma_k=\sigma_k$ for all $k\in \rho(x)^c$ one has that $\cC(\bar \sigma)$ is a cell with $x\in\overline{\cC(\bar \sigma)}$. 
	If $\rho(x)=\emptyset$ there is nothing to show and if $\rho(x)\neq \emptyset$ we get with assumption (\ref{eq:assurep}) that
	$$
	\dim\Bigl( \bigcap\limits_{j \in \rho(x)} S_j \Bigr) =d_0-\# \rho(x)
	$$
	so that the family $(\vecn_j:j\in\rho(x))$ is linearly independent. Therefore, the matrix $$M=(\vecn_j)_{j\in \rho(x)}\in \R^{n\times \rho(x)}$$
	has full rank and its adjoint $M^\dagger$ is surjective. Now let $\bar \sigma\in\{\pm1\}^n$ with $\bar \sigma_k=\sigma_k$ for all $k\in \rho(x)^c$ and pick $y\in \R^n$ with
	$$
	M^\dagger y= (\bar \sigma_j)_{j\in \rho(x)}.
	$$
	We show that for all sufficiently small $\eps>0$, $z_\eps=x+\eps y$ lies in $\cC(\bar \sigma)$. Indeed, for $k\in \rho(x)^c$ one has
	$$
	\bar \sigma_k ( \vecn_k \cdot x) >\bar \sigma_k o_k
	$$
	and by continuity, the inequality remains true for $z_\eps$ in place of $x$ as long as $\eps$ is sufficiently small. Moreover, for $k\in \rho(x)$ one has
	$$
	\bar \sigma_k( \vecn_k\cdot z_\eps ) = \bar \sigma_k \bigl(\underbrace{( \vecn_k \cdot x)}_{=o_k} +\eps \underbrace{(\vecn_k \cdot  y )}_{(M^\dagger y)_k=\bar \sigma_k}\bigr)>\bar \sigma_k o_k.
	$$
	Hence, for small $\eps>0$, one has $z_\eps\in \cC(\bar \sigma)$.

	We will also consider the reduced system, where we omit the last constraint and set  for $\tau\in\{\pm1\}^{n-1}$ 
	$$
	\cC'(\tau)=\{x\in \R^{d_0}: \tau_k (\vecn_k \cdot  x) >\tau_k o_k\text{ for all }k=1,\dots,n-1\}.
	$$
	We call $\cC'(\tau)$ a cell of the reduced system, if it is non-empty.
	
	Suppose that $\cC'(\tau)$ is a cell that intersects $S_n$. Since $\cC'(\tau)$ is an open set there exists an $x \in S_n$ which is also an element of $\overline {\cC(\tau, 1)}$ and $\overline {\cC(\tau,-1)}$. In particular, both are cells and we set
	$$
	Df_\tau=\nabla f_{(\tau,1)}-\nabla f_{(\tau,-1)}.
	$$
	
	We show that for a cell $\cC'(\tau)$ of the reduced system,  $x\in \overline {\cC'(\tau)}\cap S_k\cap S_n$ with $k\in\{1,\dots,n-1\}$  and $\bar \tau =(1-2 \1_{\{k\}})\tau$ (the orientations  one obtains when flipping the $k$th orientation) one has
	$$
	Df_{\bar \tau} =Df_\tau.
	$$
	First, we note that for every vector $w$ in the tangent space of $S_n$ we have
	$$
	\nabla f_{\tau,1}(w)=\nabla f_{\tau,-1}(w)\text{ \ and \ }\nabla f_{\bar \tau,1}(w)=\nabla f_{\bar \tau,-1}(w),
	$$
	since $f$ is differentiable in the $w$-direction on $S_n\cap \cC'(b)$ (resp.\ $S_n\cap \cC'(\bar b))$ and the differential agrees in that direction with the one on the positive and negative side of $S_n$. This entails that
	$$
	Df_\tau(w)=0=Df_{\bar \tau}(w).
	$$
	For the same reason we have for every $w'$ in the tangential direction of $S_k$ that
	$$
	\nabla f_{\bar \tau,1}(w')=\nabla f_{ \tau,1}(w')\text{ \ and \ }\nabla f_{\bar \tau,-1}(w')=\nabla f_{\tau,-1}(w')
	$$ 
and again we get that $$Df_{\bar \tau}(w')=
\nabla f_ {\bar \tau,1}(w')-\nabla f_{\bar \tau,-1}(w')=\nabla f_ {\tau,1}(w')-\nabla f_{\tau,-1}(w')=Df_{\tau}(w').	$$
By assumption, $S_k$ and $S_n$ are not parallel hyperplanes so that we obtain with linearity that $Df_{\bar \tau}=Df_{\tau}$.

 In the next step we deduce that for every cell $\cC'(\tau)$ that intersects $S_n$ one has that $Df_\tau$ does not depend on the choice of $\tau$. To see this we let $\cC'(\tau_1)$ and $\cC'(\tau_2)$ be two cells of the reduced system that intersect $S_n$, pick $x_1\in \cC'(\tau_1)\cap S_n$ and $x_2\in\cC'(\tau_2)\cap S_n$ and consider the path $\gamma:[0,1] \to S_n$  given by 
	$$
		\gamma(t)= t x_2+(1-t)x_1.
	$$
	By definition the path does not lie in one of the hyperplanes $S_1,\dots,S_{n-1}$ and there are finitely many times $0<t_1<\ldots<t_\ell<1$  at which the path switches cells of the reduced system. For $k=0,\dots,\ell$, we denote by $\cC'(\tilde \tau_k)$ the cell that is visited by the path on the time interval $(t_k,t_{k+1})$ where $t_0=0$ and $t_{\ell+1}=1$. We note that  in $\tilde \tau_{k-1}$ and $\tilde \tau_k$  all components in $\rho(\gamma(t_k))^c$ agree  for $k=1,\dots,\ell$ and it suffices to   show that for each $k=1,\dots,\ell$, $Df_{\tilde \tau_k} =Df_{\tilde \tau_{k-1}}$.
This is obtained by applying the property that $\tilde \tau_k$ can be obtained from $\tilde \tau_{k-1}$ by finitely many flips of single orientations as analyzed in the previous step. This yields that indeed $Df_{\tau_2}=Df_{\tau_1}$.	

	 It remains to represent $f$ as a neural network. As we observed for every cell $\cC'(\tau)$ of the reduced system that intersects $S_n$, the value $\Delta_n:=Df_\tau$ is the same. Note that  
	$$
	f^{n-1}(x)= f(x)-  \Delta_n ( \vecn_ n \cdot x-o_n)_+
	$$ 
	has breaklines $S_1,\dots, S_n$ and for every cell $\cC'(\tau)$ of the reduced system that intersects $S_n$ one has
	$$
	Df^{n-1}_\tau = 0.
	$$
	Consequently, $f^{n-1}$ is affine on the cells of the reduced equation and, in particular, it has
	 breaklines $S_1,\dots,S_{n-1}$. Iteration of the argument yields $\Delta_1,\dots,\Delta_n\in\R$ 
	 such that
	 $$
	 f(x)-\sum_{k=1}^n \Delta_k ( \vecn_ k \cdot x-o_k)_+
	 $$
	has no breaklines meaning that it is affine. Since every affine function can be represented with two neurons there exists a normal vector $\vecn_{n+1}=\vecn_{n+2}$, kinks $\Delta_{n+1}=-\Delta_{n+2}$ and $b^2\in\R$ such that for $o_{n+1}=o_{n+2}=0$, one has
	$$
	f= \cN^{\vecn, o,\Delta,(1,\dots,1,-1),b^2}.
	$$
	Thus we constructed an effective network with $n+2$ neurons that represents $f$.
\end{proof}

{\bf Acknowledgement.}
Funded by the Deutsche Forschungsgemeinschaft (DFG, German Research Foundation) under Germany's Excellence Strategy EXC 2044--390685587, Mathematics Münster: Dynamics--Geometry--Structure.

\bibliographystyle{alpha}
\bibliography{Parameter_space}

\end{document}